\documentclass[twoside]{article} 
\usepackage{graphicx,amsmath, amsfonts,amsthm} 
\usepackage[authoryear]{natbib}

\newtheorem{definition}{Definition}
\newtheorem{theorem}{Theorem} 
\newtheorem{prop}{Proposition} 
\newtheorem{lemma}{Lemma}

\newtheorem{corollary}{Corollary}
\newtheorem{remark}{Remark}
\newtheorem{fact}{Fact}

\newcommand{\ZZ}{\mathcal{Z}}
\newcommand{\PP}{p}
\newcommand{\LL}{\mathcal{L}}
\newcommand{\HH}{\mathcal{H}}
\newcommand{\EE}{\mathbb{E}}

\newcommand{\Sim}[2]{\langle #1, \;#2\rangle}
\newcommand{\Dis}[2]{|| #1\,,\;\,#2 ||_\mathcal{T}}
\newcommand{\info}[2]{\mathcal{J}(#1;\,#2)}
\newcommand{\infoc}[3]{\mathcal{J}(#1;\,#2\,|\,#3)}

\usepackage{aistats2017}
\begin{document}

%

%

\twocolumn[

\aistatstitle{Uniform Generalization, Concentration, and Adaptive Learning}

\aistatsauthor{
Ibrahim Alabdulmohsin
}

\aistatsaddress{
CEMSE Division, KAUST
            Thuwal 23955, Saudi Arabia 
} ]

\begin{abstract}
One fundamental goal in any learning algorithm is to mitigate its risk for overfitting. Mathematically, this requires that the learning algorithm enjoys a small generalization risk, which is defined either in expectation or in probability. Both types of generalization are commonly used in the literature. For instance, generalization in expectation has been used to analyze algorithms, such as ridge regression and SGD, whereas generalization in probability is used in the VC theory, among others. Recently, a third notion of generalization has been studied, called \emph{uniform generalization}, which requires that the generalization risk vanishes uniformly in expectation across all bounded parametric losses. It has been shown that uniform generalization is, in fact, \emph{equivalent} to an information-theoretic stability constraint, and that it recovers classical results in learning theory. It is achievable under various settings, such as sample compression schemes, finite hypothesis spaces, finite domains, and differential privacy. However, the relationship between uniform generalization and concentration remained unknown. In this paper, we answer this question by proving that, while a generalization in expectation does not imply a generalization in probability, a \emph{uniform} generalization in expectation does imply concentration. We establish a chain rule for the uniform generalization risk of the composition of hypotheses and use it to derive a large deviation bound. Finally, we prove that the bound is tight.\end{abstract}
\section{INTRODUCTION}
One of the central questions in statistical learning theory is to establish the conditions for generalization from a finite collection of observations to the future. Mathematically, this is formalized by bounding the difference between the empirical and the true risks of a given learning algorithm $\LL:\,\cup_{m=1}^\infty\,\ZZ^m\to\HH$, where $\ZZ$ is the observation space and $\HH$ is the hypothesis space. 

Informally, suppose we have a learning algorithm $\LL$ that receives a sample $S_m=\{Z_1,\ldots, Z_m\}$, which comprises of $m$ i.i.d. observations $Z_i\sim\PP(z)$, and uses $S_m$ to  select a hypothesis $H\in\HH$. Because $H$ is selected \emph{based on} the sample $S_m$, its \emph{empirical} risk on $S_m$ is a biased estimator of its \emph{true} risk with respect to the distribution of observations $\PP(z)$. The difference between the two risks, referred to as the \emph{generalization} risk, determines the prospect of over-fitting in the learning algorithm. 

In the literature, generalization bounds are often expressed either in expectation or in probability. Let $L(\cdot; H): \ZZ\to[0,1]$ be some parametric loss function that satisfies the Markov chain $S_m\to H\to L(\cdot; H)$. We write $R_{true}(H)$ and $R_{emp}(H; S_m)$ to denote, respectively, the true and the empirical risks of the hypothesis $H$ w.r.t.  $L(\cdot; H)$:
\begin{align}
\nonumber R_{emp}(H; S_m) &= \frac{1}{m}\,\sum_{Z_i\in S_m}\, L(Z_i;\,H)\\
\label{r_emp_true_eq} R_{true}(H) &= \EE_{Z\sim\PP(z)}\,\big[L(Z;\,H)\big]
\end{align}
Then, generalization in expectation and generalization in probability are defined as follows:
\begin{definition}[Generalization in Expectation]\label{def::gen_expectation}
The expected generalization risk of a learning algorithm $\LL:\,\cup_{m=1}^\infty\,\ZZ^m\to\HH$ with respect to a parametric loss $L(\cdot; H):\ZZ\to[0,1]$ is defined by:
\begin{equation}\label{Rgen_expectation}
R_{gen}(\LL) = \EE_{S_m, H|S_m}\big[R_{emp}(H; S_m) - R_{true}(H)\big],
\end{equation}
where the empirical risk $R_{emp}(H; S_m)$ and the true risk $R_{true}(H)$ are given by Eq. (\ref{r_emp_true_eq}). A learning algorithm $\LL$ generalizes in expectation if $R_{gen}(\LL)\to 0$ as $m\to\infty$ for all distributions $\PP(z)$. 
\end{definition}

\begin{definition}[Generalization in Probability]
A learning algorithm $\LL$ generalizes in probability if for any positive constant $\epsilon>0$, we have: 
\begin{equation*}
\PP\Big\{\big|R_{true}(H) - R_{emp}(H; S_m)\big|>\epsilon\Big\} \to 0 \;\;\text{ as } m\to\infty,
\end{equation*}
where the probability is evaluated over the random choice of $S_m$ and the internal randomness of $\LL$. 
\end{definition}
Clearly, for bounded loss functions, a generalization in probability implies a generalization in expectation but the converse is not generally true.

In general, both types of generalization have been studied in the literature. For instance, generalization in probability is used in the Vapnik-Chervonenkis (VC) theory, the covering numbers, and the PAC-Bayesian framework, among others \citep{vapnik1999overview,blumer1989learnability,mcallester2003pac,bousquet2002stability,bartlett2002rademacher,stat_learn_theory_2004,audibert2007combining}. Generalization in expectation, on the other hand, was used to analyze learning algorithms, such as the stochastic gradient descent (SGD), differential privacy, and ridge regression \citep{sgd_train_fast_2015,adaptive_learning_2015,SSS2014understandML}. Its common tool is a replace-one averaging lemma, similar to the Luntz-Brailovsky theorem \citep{luntz1969estimation,vapnik2000bounds}, which relates generalization to algorithmic stability \citep{sgd_train_fast_2015,SSS2014understandML}. Generalization in expectation is often simpler to analyze, but it provides a weaker performance guarantee. 

Recently, however, a third notion of generalization has been introduced in \cite{alabdulmohsin_nips_2015}, which is called \emph{uniform} generalization. It also expresses generalization bounds in expectation, but it is stronger than the notion of generalization in Definition \ref{def::gen_expectation} because it requires that the generalization risk vanishes uniformly across all bounded parametric loss functions. Importantly, uniform generalization is shown to be \emph{equivalent} to an information-theoretic algorithmic stability constraint, and that it recovers classical results in learning theory. It has been connected to the VC dimension as well \citep{alabdulmohsin_nips_2015}. Moreover, many conditions can be shown to be sufficient for uniform generalization. These include {differential privacy}, {sample compression schemes}, {perfect generalization}, {robust generalization}, {typical generalization}, {finite description lengths}, or {finite domains}. Indeed, we  prove in Appendix \ref{appendix_uniform_gen_vs_other_notions} that all such conditions are sufficient for uniform generalization.

Unfortunately, uniform generalization bounds hold only in expectation without any concentration guarantees. This sheds doubt on the utility of the notion of uniform generalization and its information-theoretic approach of analyzing learning algorithms. For instance, we will later construct a learning algorithm that generalizes \emph{perfectly} in expectation w.r.t. a \emph{specific} parametric loss even though it does not generalize almost surely over the random draw of the sample $S_m$. Hence, generalization in expectation is insufficient to ensure that a generalization will take place in practice. 

Nevertheless, we will establish in this paper that a \emph{uniform} generalization in expectation is, in fact, sufficient for a generalization in probability to hold. Moreover, we will derive a tight concentration bound. Hence, all of the uniform generalization bounds, such as the ones derived in \citep{alabdulmohsin_nips_2015}, hold, not only in expectation but with a high probability as well. Besides,  our result provides, as far as we know, the first strong connection between the two forms of generalization in the literature. We present examples of how our concentration bound can be used to deduce new concentration results for important classes of learning algorithms, such as differential privacy. 

The proof of our concentration bound rests on a chain rule that we derive for uniform generalization, which is analogous to the chain rule of mutual information in information theory \citep{cover2012elements}. Using the chain rule, we show that learning algorithms that generalize uniformly in expectation are amenable to non-adaptive composition, which is analogous to earlier results using differential privacy, sample compression schemes, and perfect generalization \citep{dwork2013algorithmic,robust_gen_2016}. Moreover, the chain rule lends support to the \emph{information budget} framework, which was recently proposed for controlling the bias of estimators in the adaptive setting using information theory \citep{control_bias_it_2016}. 

The rest of the paper is structured as follows. We will, first, briefly outline the terminology and notation used in this paper and review the existing literature. Next, we recount the main results pertaining to uniform generalization and algorithmic stability and describe how uniform generalization differs from uniform convergence. Finally, we derive the concentration bound for uniform generalization, prove its tightness, and discuss some of its implications afterward. 

\section{Terminology and Notation} 
Throughout this paper, we will always write $\ZZ$ to denote the space of observations (a.k.a. \emph{domain}) and write $\HH$ to denote the hypothesis space (a.k.a. \emph{range}). A learning algorithm $\LL:\,\cup_{m=1}^\infty\,\ZZ^m\to\HH$ is formally treated as a stochastic map, where the hypothesis $H\in\HH$ can be a deterministic or a randomized function of the sample $S_m\in\ZZ^m$.

We consider the general setting of learning introduced by Vapnik \citep{vapnik1999overview}. In this setting, the observations $Z_i\in\ZZ$ can be instance-label pairs $Z_i=(X_i, \,Y_i)$ as in supervised learning or they can comprise of instances only as in unsupervised learning. The distinction between the two learning paradigms is irrelevant. Moreover, we allow the hypothesis $H$ to be an arbitrary random variable. For instance, $H$ can be a classifier, a regression function, a statistical query, a set of centroids, a density estimate, or an enclosing sphere. Only the relationship between the two random variables $S_m$ and $H$ matters in our analysis. 

Moreover, if $Z\sim\PP(z)$ is a random variable drawn from $\mathcal{Z}$ and $f(Z)$ is a function of $Z$, we write $\EE_{Z\sim\PP(Z)}\,f(Z)$ to denote the expected value of $f(Z)$ with respect to the distribution $\PP(z)$. Occasionally, we omit $\PP(z)$ from the subscript if it is clear from the context. If $Z$ takes its values from a finite set $S$ uniformly at random, we write ${Z\sim S}$ to denote this fact. If $X$ is a boolean random variable, then $\mathbb{I}\{X\}=1$ if and only if $X$ is {true}, otherwise $\mathbb{I}\{X\}=0$.

Finally, given two probability measures $P$ and $Q$ defined on the same space, we will write $\Sim{P}{Q}$ to denote the \emph{overlapping coefficient} between $P$ and $Q$. That is, $\Sim{P}{Q} = 1- \Dis{P}{Q}$, where $\Dis{P}{Q}= \frac{1}{2}\big|\big|P-Q\big|\big|_1$ is the total variation distance.

\section{Related Work}
Generalization can be rightfully considered as an extension to the \emph{law of large numbers}, which is one of the earliest and most important results in probability theory and statistics. Suppose we have  $m$ i.i.d. observations $S_m=\{Z_1,\ldots, Z_m\}\in\ZZ^m$ and let $f:\mathcal{Z}\to [0,1]$ be an arbitrary function. If $f$ is fixed \emph{independently} of $S_m$, then $\EE_{Z_i\sim S_m} [f(Z_i)] \to \EE_{Z\sim\PP(z)} [f(Z)]$ a.s. as $m\to\infty$. This law is generally attributed to Jacob Bernoulli, who wrote an extensive treatise on the subject published posthumously in 1713 \citep{stigler1986history}. Modern proofs include low-confidence guarantees, e.g. the Chebychev inequality, and high confidence bounds, e.g. the Chernoff method \citep{cocentrationboucheron}.

When the function $f$ depends on $S_m$, the law of large numbers is no longer valid because $f(Z_i)$ are not independent random variables. One remedy is to look into the function $F(S_m) = \EE_{Z_i\sim S_m} f(Z_i)$. For instance, the Efron-Stein-Steele lemma might be used to bound the variance of $F$, which, in turn, can be translated into a concentration bound using the Chebychev inequality \citep{cocentrationboucheron,bousquet2002stability}. Alternatively, if $F$ satisfies the \emph{bounded-difference} property, then McDiarmid's inequality yields a high-confidence guarantee \citep{cocentrationboucheron,bousquet2002stability}. 

In this paper, the same question is being addressed. However, we address it in an information-theoretic manner. We will show that if the function $f:\ZZ\to[0,1]$ (as a random variable instantiated after observing the sample $S_m$) carries little \emph{information} about any individual observation $Z_i\in S_m$, then the difference between $\EE_{Z_i\sim S_m}[f(Z_i)]$ and $\EE_{Z\sim\PP(z)}[f(Z)]$ will be small with a high probability. The measure of information used here is given by the notion of \emph{variational information} $\info{X}{Y}=1-S(X; Y)$ between the random variables $X$ and $Y$, where $S(X; Y)$ is the mutual stability introduced in \cite{alabdulmohsin_nips_2015}. Variational information is an instance of the class of \emph{informativity} measures using $f$-divergences, for which an axiomatic basis has been proposed \citep{f_divergence_csiszar,f_divergence_csiszar_2008}. 

The information-theoretic approaches of analyzing the generalization risks of learning algorithms, such as the one pursued in this paper, have found applications in adaptive data analysis. This includes the method of \cite{adaptive_learning_2015} using the notion of \emph{max-information} and the method of \cite{control_bias_it_2016} using the \emph{mutual information}. For bounded losses, uniform generalization bounds using the \emph{variational information} yield tighter results, as deduced by the Pinsker inequality \citep{pinsker_colt_2009}. In this paper, we prove that these bounds hold not only in expectation but with a high probability as well. 

As a consequence of our main theorem, concentration bounds for a given learning algorithm can be immediately deduced once we recognize that the algorithm generalizes uniformly in expectation. Examples of when this holds include having (1) a finite average description length of the hypothesis, (2) a finite VC dimension of the \emph{induced concept class}, (3) differential privacy, (4) robust generalization, (5) typical generalization, (6) bounded mutual information, and (7) finite domains. We briefly describe these settings that have been previously studied in the literature and prove their connections to uniform generalization in Appendix \ref{appendix_uniform_gen_vs_other_notions}. We also present connections between uniform generalization and learnability in Appendix \ref{appendix_learnability}. A second consequence of our work is establishing the equivalence between the notion of uniform generalization studied by \cite{alabdulmohsin_nips_2015} and the notion of \emph{robust generalization} considered more recently by \cite{robust_gen_2016}.

Besides deriving a concentration bound, we also establish that our bound is tight. This tightness result is inspired by the work of \cite{bassily2015adaptive} (Lemma 7.4) and \cite{shalev2010learnability} (Example 3), where similar results are established for differential privacy and learnability respectively. In Section \ref{sect::tightness}, we combine techniques from both works to show that our concentration bound is indeed tight. 

\section{Uniform Generalization}
First, we review the main results pertaining to uniform generalization and algorithmic stability. We only mention the key results here for completeness. The reader is referred to \cite{alabdulmohsin_nips_2015} for details.

\subsection{Uniform Generalization vs. Uniform Convergence}
The main result of \cite{alabdulmohsin_nips_2015} is the \emph{equivalence} between algorithmic stability and \emph{uniform} generalization in expectation across all bounded parametric loss functions that satisfy the Markov chain: $S_m\to H\to L(\cdot; H)$. 
\begin{definition}[Uniform Generalization]\label{def::uniform_generalization}
A learning algorithm $\LL:\,\cup_{m=1}^\infty\,\ZZ^m\to\HH$ generalizes {uniformly} if for any $\epsilon>0$,  $\exists m_0(\epsilon)>0$ such that for all distributions $\PP(z)$ on $\ZZ$, all parametric losses, and all sample sizes $m> m_0(\epsilon)$, we have $|R_{gen}(\LL)\big|\le \epsilon$, where $R_{gen}(\LL)$ is given in Eq. (\ref{Rgen_expectation}).
\end{definition}
\begin{definition}\label{def::uniform_generalization_rate}
A learning algorithm $\mathcal{L}$ generalizes uniformly at the rate $\epsilon>0$ if the expected generalization risk satisfies $|R_{gen}(\LL)|\le \epsilon$ for all distributions $\PP(z)$ on $\ZZ$ and all parametric losses. 
\end{definition}
With some abuse of terminology, we will occasionally say that a learning algorithm generalizes uniformly when it generalizes uniformly according to Definition \ref{def::uniform_generalization_rate} for some provably small $\epsilon$. Whether we are referring to Definition \ref{def::uniform_generalization} or \ref{def::uniform_generalization_rate} will be clear from the context. 

Uniform generalization is different from the classical notion of uniform convergence. To see the difference, we note that a parametric loss $L(Z;H):\ZZ\times\HH\to [0,1]$ is a function of the two random variables $Z\in\ZZ$ and $H\in\HH$. This parametric loss on the product space $\ZZ\times\HH$, sometimes called the \emph{loss class} \citep{stat_learn_theory_2004}, is a family of loss functions on $Z$ indexed by $H$. Uniform convergence, such as by using the union bound or the growth function, establishes sufficient conditions for uniform convergence to hold within the family of loss functions indexed by $H$ for \emph{a single} parametric loss. These uniform convergence guarantees are often independent of how $\LL$ works. 

By contrast, suppose that the learning algorithm $\LL$ produces a hypothesis $H$ given a sample $S_m\in\mathcal{Z}^m$ with probability $\PP_\LL(H|S_m)$, where $\PP_\LL(H|S_m)$ can be degenerate in deterministic algorithms. Then, in principle, one can compute the expected generalization risk $R_{gen}(\LL)$, defined in Eq. (\ref{Rgen_expectation}), for every possible parametric loss. This is the \emph{average} loss within each possible family of bounded loss functions indexed by $H$, averaged over the random choice of $S_m$ and the internal randomness of $\LL$. Uniform generalization establishes the conditions for $|R_{gen}(\LL)|$ to go to zero uniformly across all parametric loss functions. It is heavily dependent on how $\LL$ works.  

\subsection{Previous Results}
The main result proved in \cite{alabdulmohsin_nips_2015} is that uniform generalization is equivalent to an information-theoretic stability constraint on $\LL$. 

\begin{definition}[Mutual Stability]\label{mutual_stab_def}
The mutual stability between two random variables $X$ and $Y$ is $S(X;Y) \doteq\Sim{\PP(X)\PP(Y)}{\PP(X,Y)}$, where $\Sim{P}{Q} = 1- \Dis{P}{Q}$, and $\Dis{P}{Q}$ is the total variation distance. 
\end{definition}
\begin{definition}[Variational Information]
The variational information $\info{X}{Y}$ between the random variables $X$ and $Y$ is defined by $\info{X}{Y} = 1-S(X; Y)$.
\end{definition}
Informally speaking, $\info{X}{Y}$ measures the influence of observing the value of $X$ on the distribution of $Y$. The rationale behind this definition is revealed next. 
\begin{definition}[Algorithmic Stability]\label{algorithmic_stability_def}
Let  $\LL$ be a learning algorithm that uses $S_m=\{Z_i\}_{i=1,..,m}\sim\PP^m(z)$ to produce a hypothesis $H\in\mathcal{H}$. Let $Z_{trn}\sim S_m$ be a random variable whose value is drawn uniformly at random from the sample $S_m$. Then, the algorithmic stability of $\LL$ is defined by: $\mathbb{S}(\LL) = \inf_{\PP(z)}\; S(H;\,Z_{trn})$, where the infimum is taken over all possible distributions of observations $\PP(z)$. A learning algorithm is called stable if $\lim_{m\to\infty}\;\mathbb{S}(\LL) = 1$. 
 \end{definition}
Intuitively, a learning algorithm is stable if the influence of a \emph{single} training example vanishes as $m\to\infty$. 
\begin{theorem}[Alabdulmohsin, 2015]\label{main_theorem}
For any learning algorithm $\LL:\,\cup_{m=1}^\infty\,\ZZ^m\to\HH$, algorithmic stability (Definition \ref{algorithmic_stability_def}) is both necessary and sufficient for uniform generalization (Definition \ref{def::uniform_generalization}). In addition, $\big|R_{gen}(\LL)\big| \le \info{H}{Z_{trn}} \le 1-\mathbb{S}(\LL)$, with $R_{gen}(\LL)$  defined in Eq. (\ref{Rgen_expectation}). 
\end{theorem}

Theorem \ref{main_theorem} reveals that uniform generalization has, at least, three \emph{equivalent} interpretations: 
\begin{enumerate} 
\item \emph{Statistical Interpretation}: A learning algorithm generalizes uniformly if and only if its expected generalization risk $R_{gen}(\LL)$ vanishes as $m\to\infty$ uniformly across all bounded parametric losses.
\item \emph{Information-Theoretic Interpretation}: A learning algorithm generalizes uniformly if and only if its hypothesis $H$ reveals a \emph{vanishing} amount of information about any \emph{single} observation in $S_m$ as $m\to\infty$. This, for example, is satisfied if $H$ has a finite description length or if it is sufficiently randomized as in differential privacy. 
\item \emph{Algorithmic Interpreation}: A learning algorithm  generalizes uniformly if and only if the contribution of any \emph{single} observation on the hypothesis $H$ vanishes as $m\to\infty$. That is, a learning algorithm generalizes uniformly if and only if it is algorithmically stable. 
\end{enumerate} 

Other results have also been established in \cite{alabdulmohsin_nips_2015} including the data processing inequality, the information-cannot-hurt inequality, and the uniform generalization bound in the finite hypothesis space setting. Some of those results will be used in our proofs in this paper. 

\section{Generalization in Expectation vs. Generalization in Probability} 
The main contribution of this paper is to prove that a uniform generalization in expectation implies a generalization in probability and to derive a tight concentration bound. By contrast, a non-uniform generalization in expectation does not imply that a generalization will actually take place in practice. In addition, we will also establish a \emph{chain rule} for variational information and prove that our large-deviation bound is tight. Interestingly, our proof reveals that uniform generalization is a \emph{robust} property of learning algorithms. Specifically, adding a \emph{finite} amount of information (in bits) to a hypothesis $H$ that generalizes uniformly cannot remove its uniform generalization property.

\subsection{Non-Uniform Generalization} 
We begin by showing why a non-uniform generalization in expectation does not imply concentration\footnote{Detailed proofs are available in the supplementary file.}.

\begin{prop}\label{prop_non_uniform}
There exists a learning algorithm $\LL:\,\cup_{m=1}^\infty\,\ZZ^m\to\HH$ and a parametric loss $L(\cdot; H): \ZZ\to[0,1]$ such that the expected generalization risk is $R_{gen}(\LL)=0$ for all $m\ge 1$, but for all $m\ge 1$ we have $\PP\Big\{\big|R_{true}(H) - R_{emp}(H; S_m)\big|=\frac{1}{2}\Big\} = 1$, where the probability is evaluated over the random choice of $S_m$ and the internal randomness of $\LL$. 
\end{prop}

Proposition \ref{prop_non_uniform} shows that a generalization in expectation does not imply a generalization in probability. Importantly, it is crucial to observe that the learning algorithm constructed in the proof of Proposition \ref{prop_non_uniform} does not, in fact, generalize \emph{uniformly} in expectation. Indeed, this latter observation is not a coincidence as will be proved later in Theorem \ref{gen_prob_theorem}.

\subsection{Robustness of Uniform Generalization}
Next, we prove that uniform generalization is a robust property of learning algorithms. We will use this fact later to prove that a uniform generalization in expectation implies a generalization in probability. In order to achieve this, we begin with the following chain rule.
\begin{definition}[Conditional Variational Information]
The conditional variational information between the two random variables $A$ and $B$ given $C$ is defined by: 
\begin{equation*}
\infoc{A}{B}{C} = \EE_C\big[\Dis{\PP(A,B\,|\,C)}{\PP(A|C)\cdot\PP(B|C)}\big],
\end{equation*}
which is analogous to the notion of conditional mutual information in information theory \citep{cover2012elements}.
\end{definition}
\begin{theorem}[Chain Rule]\label{chain_rule_theorem}
Let $(H_1, \ldots, H_k)$ be a sequence of random variables. Then, for any random variable $Z$, we have: $\info{Z}{(H_1,..., H_k)} \le \sum_{t=1}^k \infoc{Z}{H_t}{(H_1,..., H_{t-1})}$
\end{theorem}
Although the chain rule above provides an upper bound, the upper bound is tight in the following sense: 
\begin{prop}
For any random variables $A, B, $ and $C$, we have $\Big|\info{A}{(B,C)} - \infoc{A}{C}{B} \Big|\le \info{A}{B}$ and $\Big|\info{A}{(B,C)} - \info{A}{B} \Big| \le \infoc{A}{C}{B}$. 
\end{prop} 
In other words, the inequality in the chain rule $\info{A}{(B,C)}\le \info{A}{B} + \infoc{A}{C}{B}$ becomes an equality if $\min\{\info{A}{B},\;\infoc{A}{C}{B}\} = 0$.  

The chain rule provides a recipe for computing the bias of estimators when we have a composition of hypotheses $(H_1,\ldots, H_k)$, whether this composition is adaptive or non-adaptive. Recently, \cite{control_bias_it_2016} proposed an \emph{information budget} framework for controlling the bias of estimators by controlling the mutual information between $H$ and the entire sample $S_m$. The proposed framework rests on the well-known chain rule for mutual information. Here, we note that the chain rule for variational information in Theorem \ref{chain_rule_theorem} lends further support to the information budget framework.

Next, we use the chain rule in Theorem \ref{chain_rule_theorem} to prove that uniform generalization is a robust property of learning algorithms. More precisely, if $K$ has a finite domain, then a hypothesis $H$ generalizes uniformly in expectation if and only if the pair $(H, K)$ generalizes uniformly in expectation.  Hence, adding any finite amount of information (in bits) to a hypothesis cannot alter its uniform generalization property\footnote{Note, by contrast, that the proof of Proposition \ref{prop_non_uniform} illustrates an example where a hypothesis $H$ may generalize perfectly in expectation w.r.t. a fixed parametric loss, but a single bit of information suffices to destroy this generalization advantage. This never occurs when $H$ generalizes uniformly since uniform generalization is a robust property.}. 
\begin{theorem}\label{robustness_theorem}
Let $\LL:\,\cup_{m=1}^\infty\,\ZZ^m\to\HH$ be a learning algorithm whose hypothesis is $H\in\HH$, which is obtained from a sample $S_m$. Let $K\in\mathcal{K}$ be a different hypothesis that is obtained from the same sample $S_m$. If $Z_{trn}\sim S_m$ is a random variable whose value is drawn uniformly at random from $S_m$, then: 
\begin{align*}
\info{Z_{trn}}{(H, K)}\le (1+\frac{|\mathcal{K}|}{2})\cdot \info{Z_{trn}}{H} + \sqrt{\frac{\log |\mathcal{K}|}{2m}}
\end{align*}
\end{theorem}
\subsection{Uniform Generalization Implies Concentration}
Theorem \ref{robustness_theorem} shows that adding a finite amount of information (in bits) cannot remove the uniform generalization property of learning algorithms. We will use this fact, next, to prove that a uniform generalization in expectation implies a generalization in probability. 

The intuition behind the proof is as follows. Suppose we have a hypothesis $H$ that generalizes uniformly in expectation but, for the purpose of obtaining a contradiction, suppose that there exists a parametric loss $L(\cdot; H)$ that does not generalize in probability. Then, \emph{adding} little information to the hypothesis $H$ will allow us to construct a \emph{different} parametric loss that does not generalize in expectation. In particular, we will only to need to know whether the empirical risk w.r.t. $L(\cdot; H)$ is greater than, approximately equal to, or is less than the true risk w.r.t. the same loss. This is described in, at most, two bits. Knowing this additional information, we can define a new parametric loss that does not generalize in expectation, which contradicts the statement of Theorem \ref{robustness_theorem}. This line of reasoning is formalized in the following theorem. 
\begin{theorem}\label{gen_prob_theorem}
Let $\LL:\,\cup_{m=1}^\infty\,\ZZ^m\to\HH$ be a learning algorithm, whose risk is evaluated using a parametric loss function $L(\cdot; H):\ZZ\to[0,1]$. Then: 
\begin{align*}
&\PP\Big\{\big|R_{emp}(H; S_m) - R_{true}(H)\big| \ge t\Big\} \\
&\le\frac{5}{2t}\Big[\info{Z_{trn}}{H} + \sqrt{\frac{\log 9}{25 m}}\Big]\le \frac{5}{2t}\Big[1-\mathbb{S}(\LL) + \sqrt{\frac{\log 9}{25 m}}\Big], 
\end{align*} 
where $\mathbb{S}(\LL)$ is the algorithmic stability of $\LL$ given  in Definition \ref{algorithmic_stability_def}, and the probability is evaluated over the random choice of $S_m$ and the internal randomness of $\LL$. In particular, if $\LL$ generalizes uniformly, i.e. $\mathbb{S}(\LL)\to 1$ as $m\to\infty$, it generalizes in probability for any chosen parametric loss. 
\end{theorem}

The same proof technique used in Theorem \ref{gen_prob_theorem} also implies the following concentration bound, which is useful when $I(H; S_m) = o(m)$. The following bound compares well with the bound derived in \cite{control_bias_it_2016} using properties of sub-Gaussian loss functions. 
\begin{prop}\label{prop::mutual_info_bound}
Let $\LL:\,\cup_{m=1}^\infty\,\ZZ^m\to\HH$ be a learning algorithm, whose risk is evaluated using a parametric loss function $L(\cdot; H):\ZZ\to[0,1]$. Then: 
\begin{align*}
\PP\Big\{\big|R_{emp}(H; S_m) - R_{true}(H)\big| \ge t\Big\}\le\frac{1}{t}\sqrt{\frac{I(S_m; H) + 3}{2m}}, 
\end{align*} 
\end{prop}
Note that having a bounded mutual information, i.e. $I(S_m; H)=o(m)$, which is the setting recently considered in the work of \cite{control_bias_it_2016}, is sufficient for uniform generalization to hold.

\subsection{Implications}
\subsubsection{Concentration}
In \cite{alabdulmohsin_nips_2015}, it was shown that the notion of uniform generalization allows us to reason about learning algorithms in pure information-theoretic terms. This is because uniform generalization is equivalent to an information-theoretic algorithmic stability constraint on learning algorithms. For example, the data processing inequality implies that one can improve the uniform generalization risk by either post-processing the hypothesis, such as sparsification or decision tree pruning, or by pre-processing training examples, such as by introducing noise, Tikhonov regularization, or dropout. Needless to mention, both are common techniques in machine learning. In addition, uniform generalization recovers classical results in learning theory, such as the generalization bounds in the finite hypothesis space setting and finite domains \citep{alabdulmohsin_nips_2015}. However, such conclusions previously held only \emph{in expectation}. 

The most important implication of Theorem \ref{gen_prob_theorem} is to establish that such conclusions actually hold with a high probability as well. In addition, the concentration bound derived in Theorem \ref{gen_prob_theorem} shows that algorithmic stability $\mathbb{S}(\LL)$ not only controls the generalization risk of $\LL$ in expectation, i.e. due to its equivalence with uniform generalization, but it also controls the rate of convergence in probability. This brings us to the following important remark: 
\begin{remark}
By improving algorithmic stability, we improve both the expectation of the generalization risk \emph{and} its variance. 
\end{remark}

Besides, Theorem \ref{gen_prob_theorem} can be useful in deriving new concentrations bounds for important classes of learning algorithms once we recognize the existence of uniform generalization. We illustrate this technique on \emph{differential privacy} next.

\subsection{Differential Privacy}
Differential privacy addresses the goal of obtaining useful information about the sample $S_m$ as a whole without revealing a lot of information about each individual observation in the sample \citep{dwork2013algorithmic}. It closely resembles the notion of algorithmic stability proposed in \cite{alabdulmohsin_nips_2015} because a learning algorithm is stable according to the latter definition if and only if the posterior distribution of an individual observation $Z_{trn}$ in the sample $S_m$ becomes arbitrarily close, in the total variation distance, to the prior distribution $\PP(Z_{trn})$ as $m\to\infty$. Indeed, differential privacy is a stronger privacy guarantee. 

\begin{definition}[Dwork \& Roth, 2013]
A randomized learning algorithm $\LL:\,\cup_{m=1}^\infty\,\ZZ^m\to\HH$ is $(\epsilon,\delta)$ differentially private if for any $\mathcal{O}\subseteq \HH$ and any two samples $S$ and $S'$ that differ in one observation only, we have: 
\begin{equation*}
\PP(H\in \mathcal{O}\;|\;S) \le e^\epsilon\cdot \PP(H\in \mathcal{O}\;|\;S') + \delta
\end{equation*}
\end{definition}

Concentration bounds for differential privacy have been derived, such as the recent work of \cite{bassily2015adaptive}. Nevertheless, we remark here that Theorem \ref{gen_prob_theorem} can be used to derive a new concentration bound for differential privacy. Comparing our bound with the lower bound of Lemma 7.4 in \cite{bassily2015adaptive} reveals that the dependence on $\delta$ and $t$ is tight up to a constant factor.
\begin{corollary}\label{diff_privacy_lemma}
If a learning algorithm $\LL:\,\cup_{m=1}^\infty\,\ZZ^m\to\HH$ is $(\epsilon,\delta)$ differentially private, then: $\PP\Big\{\big|R_{emp}(H; S_m) - R_{true}(H)\big| \ge t\Big\}\le\frac{5}{4t}\Big[e^\epsilon-1+\delta+\sqrt{\frac{2\log 9}{25m}}\Big]$.
\end{corollary}
Not surprisingly, the differential privacy parameters $(\epsilon,\delta)$ control the generalization risk of differential privacy, with the quantity $(e^\epsilon-1+\delta)$ acting a role that is analogous to the standard error. 

\subsubsection{Equivalnce with Robust Generalization}
Another implication of the concentration bound in Theorem \ref{gen_prob_theorem} is establishing the \emph{equivalence} between the notion of uniform generalization and the notion of \emph{robust generalization} studied in \cite{robust_gen_2016}. 

\begin{definition}[Robust Generalization]
A learning algorithm $\mathcal{L}$ is $(\epsilon,\delta)$ robustly generalizing if for all distribution $\PP(z)$ on $\mathcal{Z}$ and any binary-valued parametric loss function  $L(\cdot; H):\mathcal{Z}\to\{0,1\}$ that satisfies the Markov chain $S_m\to H\to L(\cdot; H)$, we have with a probability of at least $1-\zeta$ over the choice of $S$ that: 
\begin{equation*}
\PP\Big\{\big|\EE_{Z\sim \PP(z)} L(Z; H) - \frac{1}{m} \sum_{Z_i\in S} L(Z_i; H)\big| \le \epsilon\Big\}\ge 1-\gamma,
\end{equation*} 
for some $\gamma,\zeta$ such that $\delta=\gamma+\zeta$\footnote{The original definition proposed in \cite{robust_gen_2016} states that the probability is evaluted over any \lq\lq adversary" that takes the hypothesis $H$ as input to produce a loss function $L(\cdot; H)$. However, this is equivalent to the Markov chain $S_m\to H\to L(\cdot;H)$.}. 
\end{definition}

In the following theorem, we prove that robust generalization is equivalent to uniform generalization.
\begin{corollary}\label{robust_gen_equiv_uniform}
If a learning algorithm $\mathcal{L}$ is $(\epsilon,\delta)$ robustly generalizing, then it generalizes uniformly at the rate $\epsilon+\delta$. Conversely, if a learning algorithm generalizes uniformly with rate $\tau$, then it is $(\epsilon, \gamma)$ robustly generalizing with $\gamma = (5/2)(\tau + \sqrt{\log 9/(25m)})/\epsilon$. Moreover, if $\mathbb{S}(\LL)\to 0$ as $m\to \infty$, then both $\gamma$ and $\epsilon$ can be made arbitrarily close to zero using a sufficiently large sample size $m$.
\end{corollary}

\subsection{Tightness Result}\label{sect::tightness}
Finally,we note that the concentration bound has a linear dependence on the algorithmic stability term $1-\mathbb{S}(\LL)$ or, in a distribution-dependent manner, on the variational information $\info{Z_{trn}}{H}$. Typically, $\info{Z_{trn}}{H}= \Theta(1/\sqrt{m})$. By contrast, the VC bound provides an exponential decay for supervised classification tasks \citep{vapnik1999overview,shalev2010learnability}. This raises the question of whether or not the concentration bound in Theorem \ref{gen_prob_theorem} can be improved. In this section, we prove that the bound is actually tight. The following theorem is inspired by the work of \citep{bassily2015adaptive} (Lemma 7.4) and \cite{shalev2010learnability} (Example 3), who established similar results for differential privacy and learnability respectively. 

\begin{theorem}\label{tight_theorem}
For any rational $0<t<1$, there exists a learning algorithm $\LL:\,\cup_{m=1}^\infty\,\ZZ^m\to\HH$, a sample size $m$, a distribution $\PP(z)$,  and a parametric loss $L(\cdot; H):\ZZ\to[0,1]$ such that $\LL$ generalizes uniformly in expectation and it satisfies: 
\begin{equation*}
\PP\Big\{\big|R_{emp}(H; S_m) - R_{true}(H)\big| = t \Big\} = \frac{\info{Z_{trn}}{H}}{t}
\end{equation*}
\end{theorem}
Theorem \ref{tight_theorem} shows that, without making any additional assumptions beyond that of uniform generalization, the concentration bound in Theorem \ref{gen_prob_theorem} is tight  up to constant factors. Essentially, the only difference between the upper and the lower bounds is a vanishing $O(1/\sqrt{m})$ term that is \emph{independent} of $\LL$.


\section{Conclusions}
Uniform generalization in expectation is a notion of generalization that is equivalent to an information-theoretic algorithmic stability constraint on learning algorithms. In this paper, we proved that whereas generalization in expectation does not imply a generalization in probability, a uniform generalization in expectation implies a generalization in probability and we derived a tight concentration bound. The bound reveals that algorithmic stability improves both the expectation of the generalization risk and its variance. Hence, by constraining the \lq\lq amount" of  information that a hypothesis can carry about any \emph{individual} training example or, equivalently, by limiting the \lq\lq size" of the contribution of any individual training example on the final hypothesis, the learning algorithm is guaranteed to generalize well with a high probability. Furthermore, we proved a chain rule for variational information, which revealed that uniform generalization is a robust property of learning algorithms. Finally, we proved that the concentration bound is tight. 

\appendix 
\section{Relations to Other Notions of Generalization \& Stability}\label{appendix_uniform_gen_vs_other_notions}
The connection between differential privacy and uniform generalization is summarized as follows.
\begin{prop}\label{diff_privacy_lemma}
Let $\mathcal{L}$ be a ($\epsilon, \delta$)-differentially private learning algorithm. Let $Z_{trn}\sim S$ be a single training example and let $H\sim \PP_\mathcal{L}(h|S)$ be the hypothesis produced by $\mathcal{L}$. Then $\info{Z_{trn}}{H} \le \frac{e^\epsilon-1+\delta}{2}$.
\end{prop}

Next, it can be shown that perfect generalization implies differential privacy \citep{robust_gen_2016} so it implies uniform generalization. Also, sample compression implies robust generalization \citep{robust_gen_2016}, which in turn implies a uniform generalization. Moreover, typical stability \citep{typicality_based_stability} is equivalent to perfect generalization when the observations are drawn i.i.d., so it implies uniform generalization as well.

The proof that a bounded mutual information, i.e. having $I(S_m; H) = o(m)$, implies uniform generalization is a direct consequence of the concentration bound in Proposition \ref{prop::mutual_info_bound}. 

Finally, the proofs that a finite hypothesis space, a finite VC dimension in the induced concept class, and a finite domain are each sufficient for uniform generalization to hold are provided in \cite{alabdulmohsin_nips_2015}. 

\section{Uniform Generalization and Learnability}\label{appendix_learnability}
\subsection{Consistency of Empirical Risk Minimization} 
Uniform generalization is a sufficient condition for the consistency of empirical risk minimization (ERM). Suppose we have an ERM learning algorithm, whose hypothesis is denoted $H_{ERM}$. We have by definition:
\begin{align*}
R_{emp}(\LL) &= \EE_S\EE_{H|S} [R_{emp}(H; S)] = \EE_S [\min_{h\in\mathcal{H}} R_{emp}(h; S)]\\
&\le \min_{h\in\mathcal{H}} \big[\EE_S R_{emp}(h; S)\big] = \min_{h\in\mathcal{H}} R_{true}(h)\\ 
&= R_{true}(h^\star),
\end{align*}
where $h^\star$ is the optimal hypothesis. However, the true risk of $\LL$ satisfies: 
\begin{align*}
R_{true}(\LL) - R_{true}(h^\star) \le R_{true}(\LL) - R_{emp}(\LL)
\end{align*}
Thus, algorithmic stability of ERM implies consistency because $\info{Z_{trn}}{H_{ERM}}$ provides an upper bound on $|R_{true}(\LL)-R_{emp}(\LL)\big|$. In fact, because $R_{true}(H_{ERM})-R_{true}(h^\star)\ge 0$, we have by the Markov inequality: 
\begin{equation*}
\PP\Big\{R_{true}(H_{ERM})-R_{true}(h^\star) \ge t\Big\} \le \frac{\info{Z_{trn}}{H_{ERM}}}{t}
\end{equation*}
\subsection{Sample Compression and Learnability}
Moreover, recent results on the connection between sample compression schemes and learnability \citep{supervised_compression} reveal that any learnable hypothesis space is learnable using an algorithm that generalizes uniformly in expectation, with only a logarithmic increase in the sample complexity. Because sample compression schemes satisfy robust generalization \citep{robust_gen_2016}, they generalize uniformly in expectation. 

\section{Implications of the Chain Rule} 
The chain rule in Theorem \ref{chain_rule_theorem} along with the fact that for any random variables $X,Y,Z$, we have $\info{X}{Y}\le \info{X}{Y,Z}$ can both be used to derive many results. We illustrate this with two examples here. First, suppose we have the Markov chain $A\to B \to C$. By the chain rule, we have: 
\begin{equation*}
\info{A}{C} \le \info{A}{(B,C)} \le \info{A}{B} + \infoc{A}{C}{B}
\end{equation*} 
Since $A\to B\to C$ implies $\infoc{A}{C}{B}=0$ by the Markov property, we conclude that $\info{A}{C}\le \info{A}{B}$. This is the \emph{data processing} inequality. In general, we have the triangle-like inequality: 
\begin{equation}\label{triangle_inequality}
\info{X}{Y} \le \info{X}{Z} + \infoc{X}{Y}{Z}
\end{equation}
Second, if $A\to B\to C$, we have: 
\begin{equation*}
\info{A}{B} \le  \info{A}{(B,C)} \le \info{A}{B} + \infoc{A}{C}{B}
\end{equation*}
Since $\infoc{A}{C}{B}=0$, we conclude that $A\to B\to C$ implies $\info{A}{(B,C)} = \info{A}{B}$. Both results were used previously, and were proved using different methods in \cite{alabdulmohsin_nips_2015}.

\section{Proof of Proposition 1}
Let $\mathcal{Z}=[0,1]$ be an instance space with a continuous marginal density $\PP(z)$ (hence, has no atoms) and let $\mathcal{Y}=\{-1, +1\}$ be the target set. Let $h^\star:\mathcal{Z}\to\{-1, +1\}$ be some \emph{fixed} predictor, such that $\PP\{h^\star(Z)=1\} = \frac{1}{2}$, where the probability is evaluated over the random choice of $Z\in\mathcal{Z}$. In other words, the marginal distribution of the labels predicted by $h^\star(\cdot)$ is uniform\footnote{These assumptions are satisfied, for example, if $\PP(z)$ is uniform in $[0,1]$ and $h^\star(Z)=\mathbb{I}\{Z<\frac{1}{2}\}$.}. 

Next, let the hypothesis space $\HH$ be the set of predictors from $\mathcal{Z}$ to $\{-1, +1\}$ that output a label in $\{-1,+1\}$ uniformly at random everywhere in $\mathcal{Z}$ except at a finite number of points. Therefore, the hypothesis $H:\mathcal{Z}\to \{-1, +1\}$ selected by the learning algorithm is a predictor. Define the parametric loss by $L(Z; H) = \mathbb{I}\big\{H(Z)\neq h^\star(Z)\big\}$. 

Next, we construct a learning algorithm $\LL$ that generalizes perfectly in expectation but it does not generalize in probability. The learning algorithm $\LL$ simply picks one of $H_{S_m}^0(\cdot)$ or $H_{S_m}^1(\cdot)$ with equal probability, where: 
\begin{align*}
H_{S_m}^0(Z) &=\begin{cases} - h^\star(Z) & \text{ if } Z\in S_m \\ \text{Uniform}(-1, +1) & \text{ if } Z \notin S_m \end{cases}\\
H_{S_m}^1(Z) &=\begin{cases}  h^\star(Z) & \text{ if } Z\in S_m \\ \text{Uniform}(-1, +1) & \text{ if } Z \notin S_m \end{cases}
\end{align*}
Because $\ZZ$ is uncountable, where the probability of seeing the same observation $Z$ twice is zero, $R_{true}(H)=\frac{1}{2}$ for this learning algorithm. Thus: 
\begin{equation*}
R_{gen}(\LL) = \EE_{S_m,H} \big[R_{emp}(H; S_m) - R_{true}(H)\big] = 0
\end{equation*}
However, the empirical risk for any $S_m$ satisfies $R_{emp}(H; S_m)\in\{0, 1\}$ while the true risk always satisfies $R_{true}(H)=\frac{1}{2}$, as mentioned earlier. Hence, the statement of the proposition follows. 

Finally, we prove that the algorithm does not generalize uniformly in expectation. There are, at least, two ways of showing this. The first approach is to use the equivalence between uniform generalization and algorithmic stability as stated in Theorem 1. Given the hypothesis $H\in\{H_{S_m}^0,\;H_{S_m}^1\}$ learned by the algorithm constructed here, the marginal distribution of an individual training example $\PP(Z_{trn}|H)$ is uniform over the sample $S_m$. This follows from the fact that the hypothesis $H$ has to encode the entire sample $S_m$. However, the probability of seeing the same observation twice is zero (by construction). Hence, $\Dis{\PP(Z_{trn})}{\PP(Z_{trn}|H)}=1$ for all $H$. This shows that $\mathbb{S}(\LL)=0$ for all $m\ge 1$, and the learning algorithm is not stable. Therefore, by Theorem 1, it does not generalize uniformly. Note that we used the information-theoretic interpretation of uniform generalization. 

The second approach is to use the statistical interpretation of uniform generalization. Let $H\in\{H_{S_m}^0,\;H_{S_m}^1\}$ be the hypothesis inferred by the learning algorithm above, and consider the following \emph{different} parametric loss: 
\begin{equation*}
L(Z; H_{S_m}^k) = \mathbb{I}\big\{(-1)^{k+1} H_{S_m}^k(Z)\neq h^\star(Z)\big\}
\end{equation*}
In other words, we flip the predictions of $H^k_{S_m}$ if $k=0$ and measure the misclassification loss afterwards. Note that this is a parametric loss; it has a bounded range and satisfies the Markov chain $S_m\to H\to L(\cdot; H)$. However, the expected generalization risk w.r.t. this parametric loss is $R_{gen}(\LL) = \frac{1}{2}$ for all $m\ge 1$ because $R_{emp}(\LL)=0$ w.r.t. to this loss. Therefore, $\LL$ does not generalize uniformly in expectation\footnote{We remark here that the Markov inequality cannot be used to provide a concentration bound for the learning algorithm in Proposition 1 even if the expected generalization risk goes to zero because the quantity $R_{emp}(H; S_m)-R_{true}(H)$ is not guaranteed to be non-negative. Indeed, this is precisely why the learning algorithm $\LL$ constructed here generalizes in expectation but not in probability.}. 

\section{Proof of Theorem 2}
We will first prove the inequality when $k=2$. First, we write by definition: 
\begin{equation*}
\info{Z}{(H_1, H_2)} = \Dis{\PP(Z, H_1, H_2)}{\PP(Z)\, \PP(H_1, H_2)}
\end{equation*}
Using the fact that the total variation distance is related to the $\ell_1$ distance by $\Dis{P}{Q} = \frac{1}{2} ||P-Q||_1$, we have: 
\begin{align*}
&\info{Z}{(H_1, H_2)} = \frac{1}{2} \big|\big|\;\PP(Z, H_1, H_2)-\PP(Z)\, \PP(H_1, H_2) \;\big|\big|_1\\
&=\frac{1}{2} \big|\big|\; \PP(Z, H_1)\,\PP(H_2|Z,H_1)-\PP(Z)\,\PP(H_1)\, \PP(H_2|H_1)\;\big|\big|_1\\
&=\frac{1}{2} \big|\big|\; \big[\PP(Z, H_1)-\PP(Z)\,\PP(H_1)\big]\cdot \PP(H_2|H_1) \\ &\quad\quad\quad\quad+\PP(Z,H_1)\,\cdot\big[\PP(H_2|Z,H_1) -\PP(H_2|H_1)\big]\;\big|\big|_1
\end{align*}
Using the triangle inequality: 
\begin{align*}
&\info{Z_{trn}}{(H_1, H_2)}\\ &\quad\le \frac{1}{2}\Big|\Big|\big[\PP(Z, H_1)-\PP(Z)\,\PP(H_1)\big]\cdot \PP(H_2|H_1)\Big|\Big|_1 \\ 
&\quad\quad\quad+ \frac{1}{2}\;\Big|\Big|\PP(Z,H_1)\cdot\big[\PP(H_2|Z,H_1) -\PP(H_2|H_1)\big]\Big|\Big|_1
\end{align*}
The above inequality is interpreted by expanding the $\ell_1$ distance into a sum of absolute values of terms in the product space $\ZZ\times\HH_1\times\HH_2$, where $H_k\in\HH_k$. Next, we bound each term on the right-hand side separately. For the first term, we note that: 
\begin{align}\label{chain_rule_1}
\nonumber \frac{1}{2}& \big|\big|\; \big[\PP(Z, H_1)-\PP(Z)\,\PP(H_1)\big]\cdot \PP(H_2|H_1)\;\big|\big|_1\\ &\quad= \frac{1}{2} \big|\big|\; \PP(Z, H_1)-\PP(Z)\,\PP(H_1)\;\big|\big|_1= \info{Z}{H_1}
\end{align}
The equality holds by expanding the $\ell_1$ distance and using the fact that $\sum_{H_2} \PP(H_2|H_1) = 1$. 

However, the second term can be re-written as: 
\begin{align}\label{chain_rule_2}
\nonumber\frac{1}{2}& \big|\big|\; \PP(Z,H_1)\,\cdot\big[\PP(H_2|Z,H_1) -\PP(H_2|H_1)\big]\;\big|\big|_1\\
\nonumber&=\frac{1}{2}  \big|\big|\;\PP(H_1)\cdot\big[\PP(H_2,Z|H_1) -\PP(Z|H_1)\,\PP(H_2|H_1)\big]\;\big|\big|_1\\
\nonumber&=\EE_{H_1}\big[\Dis{\PP(H_2,Z|H_1)}{\PP(Z|H_1)\,\PP(H_2|H_1)}\big] \\ 
&=\infoc{Z}{H_2}{H_1}
\end{align}
Combining Eq. (\ref{chain_rule_1}) and (\ref{chain_rule_2}) yields the inequality: 
\begin{equation}\label{chain_rule_k2}
\info{Z}{(H_1,H_2)} \le \info{Z}{H_1}+\infoc{Z}{H_2}{H_1}
\end{equation} 
Next, we use Eq. (\ref{chain_rule_k2}) to prove the general statement for all $k\ge 1$. By writing: 
\begin{align*}
\info{Z}{(H_1,\ldots, H_k)}&\le \infoc{Z}{H_k}{(H_1,\ldots, H_{k-1})}\\ &\quad\quad+ \info{Z}{(H_1,\ldots, H_{k-1})}
\end{align*}
Repeating the same inequality on the last term on the right-hand side yields the statement of the theorem. 

\section{Proof of Proposition 2}
We will use the following fact \citep{alabdulmohsin_nips_2015}: 
\begin{fact}[Information Cannot Hurt]\label{info_cant_hurt_fact} For any random variables $X,Y,Z$:
\begin{equation*}
\info{X}{Y} \le \info{X}{(Y,Z)}
\end{equation*}
\end{fact}
Now, by the triangle inequality: 
\begin{align*}
\infoc{A}{C}{B} &= \EE_B \Dis{\PP(A|B)\cdot \PP(C|B)}{\PP(A,C|B)}\\
&=\EE_{A,B} \Dis{\PP(C|B)}{\PP(C|A,B)}\\
&\le \EE_{A,B} \Dis{\PP(C|B)}{\PP(C)} \\ &\quad\quad+ \EE_{A,B} \Dis{\PP(C)}{\PP(C|A,B)}\\
&= \EE_B  \Dis{\PP(C|B)}{\PP(C)} \\ &\quad\quad + \EE_{A,B} \Dis{\PP(C)}{\PP(C|A,B)}\\
&= \info{B}{C} + \info{C}{(A,B)}
\end{align*}
Therefore: 
\begin{equation*}
\info{C}{(A,B)} \ge \infoc{A}{C}{B}  -  \info{B}{C}
\end{equation*} 
Combining this with the following chain rule of Theorem 2: 
\begin{equation*}
\info{C}{(A,B)} \le \infoc{A}{C}{B}  +  \info{B}{C}
\end{equation*} 
yields: 
\begin{equation*}
\Big|\info{C}{(A,B)} - \infoc{A}{C}{B}  \Big| \le \info{B}{C}
\end{equation*}
Or equivalently:
\begin{equation}\label{eq::prop2_eq1}
\Big|\info{A}{(B,C)} - \infoc{A}{C}{B}  \Big| \le \info{A}{B}
\end{equation}

To prove the other inequality, we use Fact \ref{info_cant_hurt_fact}. We have: 
\begin{equation*}
\info{A}{B} \le \info{A}{(B,C)} \le \info{A}{B} + \infoc{A}{C}{B},
\end{equation*}
where the first inequality follows from Fact \ref{info_cant_hurt_fact} and the second inequality follows from the chain rule. Thus, we obtain the desired bound: 
\begin{equation}\label{eq::prop2_eq2}
\Big| \info{A}{(B,C)} - \info{A}{B}\Big| \le  \infoc{A}{C}{B}
\end{equation}
Both Eq. \ref{eq::prop2_eq1} and Eq. \ref{eq::prop2_eq2} imply that the chain rule is tight. More precisely, the inequality can be made arbitrarily close to an equality when one of the two terms in the upper bound is chosen to be arbitrarily close to zero.

\section{Proof of Theorem 3}
We will use the following fact:
\begin{fact}\label{distance_loss_fact}
Let $f:\mathcal{X}\to [0,1]$ be a function with a bounded range in the interval $[0,1]$. Let $\PP_1(x)$ and $\PP_2(x)$ be two different probability measures defined on the same space $\mathcal{X}$. Then:
\begin{equation*}
\Big|\EE_{X\sim\PP_1(x)} f(X) - \EE_{X\sim\PP_2(x)} f(X) \Big|\le \Dis{\PP_1(x)}{\PP_2(x)}
\end{equation*}
\end{fact} 

First, consider the following scenario. Suppose a learning algorithm $\LL$ generates a hypothesis $H\in\HH$ from some marginal distribution $\PP(h)$ \emph{independently} of the sample $S_m$. Afterward, a sample $S_m\in\mathcal{Z}^m$ is observed, which comprises of $m$ i.i.d. observations. Then, $\LL$ selects $K\in\mathcal{K}$ according to $\PP(k|H,S_m)$. 

In this scenario, we have: 
\begin{equation*}
\info{Z_{trn}}{(H, K)} = \infoc{Z_{trn}}{K}{H},
\end{equation*}
where the equality follows from the chain rule in Theorem 2, the statement of Proposition 2, and the fact that $\info{Z_{trn}}{H}=0$. The conditional variational information is written as: 
\begin{align*}
&\infoc{Z_{trn}}{K}{H} \\
&\quad\quad= \EE_H \Dis{\PP(Z_{trn})\cdot \PP(K|H)}{\PP(Z_{trn},K|H)},
\end{align*}
where we used the fact that $\PP(Z_{trn}|H)=\PP(Z_{trn})$. Next, by marginalization, the conditional distribution $\PP(K|H)$ is given by: 
\begin{align*}
 \PP(K|H) &= \EE_{Z_{trn}'|H} [\PP(K|Z_{trn}',H)] \\ 
&= \EE_{Z_{trn}'} [\PP(K|Z_{trn}',H)].
\end{align*}
where the expectation is taken with respect to the marginal distribution of observations $\PP(z)$. Similarly: 
\begin{align*}
 \PP(Z_{trn},K|H) &= \PP(Z_{trn}|H)\cdot \PP(K|Z_{trn},H)\\
&= \PP(Z_{trn})\cdot \PP(K|Z_{trn},H)
\end{align*}
Therefore: 
\begin{align*}
&\infoc{Z_{trn}}{K}{H} \\
&\quad\quad=\EE_H\EE_{Z_{trn}}\Dis{\EE_{Z_{trn}'}\PP(K|Z_{trn}',H)}{\PP(K|Z_{trn},H)}
\end{align*}
Next, for every value of $H$ that is generated independently of the sample $S_m$, the variational information between $Z_{trn}\sim S_m$ and $K\in\mathcal{K}$ can be bounded using Theorem 3 in \citep{alabdulmohsin_nips_2015}. This follows because $H$ is selected independently of the sample $S_m$, and, hence, the i.i.d. property of the observations $Z_i$ continue to hold. Therefore, we obtain: 
\begin{align}\label{robust_eq_1}
\nonumber\EE_H\EE_{Z_{trn}}&\Dis{\EE_{Z_{trn}'}\PP(K|Z_{trn}',H)}{\PP(K|Z_{trn},H)}\\ 
\nonumber&= \infoc{Z_{trn}}{K}{H}\\ &\le \sqrt{\frac{|\mathcal{K}|}{2m}}
\end{align} 
Because $\PP(K|Z_{trn},H)$ is arbitrary, the above bound holds for any distribution of observations $\PP(z)$, any distribution $\PP(h)$, and any family of conditional distributions $\PP(k|Z_{trn},H)$.  

Next, we return to the original setting where both $H\in\mathcal{H}$ and $K\in\mathcal{K}$ are chosen according to the sample $S_m$. We have: 
\begin{align}\label{robust_eq_2}
\nonumber&\infoc{Z_{trn}}{K}{H}\\
\nonumber&\quad= \EE_H \Dis{p(Z_{trn}|H)\cdot p(K|H)}{p(Z_{trn},K|H)}\\
\nonumber&\quad=  \EE_{H,Z_{trn}} \Dis{p(K|H)}{p(K|Z_{trn},H)}\\
\nonumber&\quad= \EE_{H,Z_{trn}} \Dis{\EE_{Z_{trn}'|H} [p(K|Z_{trn}',H)]}{p(K|Z_{trn},H)}\\
\nonumber&\quad\le\EE_{H,Z_{trn}} \Dis{\EE_{Z'_{trn}|H} [p(K|Z_{trn}',H)]}{\EE_{Z_{trn}'} [p(K|Z_{trn}',H)]} \\ 
&\quad\quad + \EE_{H,Z_{trn}}  \Big|\Big| \EE_{Z'_{trn}} [p(K|Z'_{trn},H)] - p(K|Z_{trn},H)\Big|\Big|_1
\end{align}

We bound the second term in Eq. \ref{robust_eq_2} using our earlier result in Eq. \ref{robust_eq_1}.  

Next, we would like to bound the first term. Using the fact that the total variation distance is related to the $\ell_1$ distance by $\Dis{P}{Q} = \frac{1}{2}||P-Q||_1$, we have: 
\begin{align}\label{robust_eq_3}
\nonumber&\EE_{H,Z_{trn}} \Dis{\EE_{Z'_{trn}|H} [p(K|Z_{trn}',H)]}{\EE_{Z_{trn}'} [p(K|Z_{trn}',H)]} \\
\nonumber&=\EE_{H}\; \Dis{\EE_{Z'_{trn}|H} [p(K|Z_{trn}',H)]}{\EE_{Z_{trn}'} [p(K|Z_{trn}',H)]} \\
\nonumber&= \frac{1}{2}\;\EE_{H}\sum_{K\in\mathcal{K}} \Big|\EE_{Z'_{trn}|H} [p(K|Z_{trn}',H)] - \EE_{Z'_{trn}} [p(K|Z_{trn}',H)] \Big|\\
\nonumber&\le \frac{1}{2}\;\EE_{H}\sum_{K\in\mathcal{K}}  \Dis{\PP(Z_{trn}'|H)}{\PP(Z_{trn}'}\\
\nonumber&= \frac{1}{2}\;\sum_{K\in\mathcal{K}} \EE_H \Dis{\PP(Z_{trn}'|H)}{\PP(Z_{trn}'} \\ 
&= \frac{|\mathcal{K}|}{2}\; \info{Z_{trn}}{H}
\end{align}
Here, the inequality follows from Fact \ref{distance_loss_fact}.

Combining all results in Eq. \ref{robust_eq_1}, \ref{robust_eq_2}, and \ref{robust_eq_3}: 
\begin{equation}
\infoc{Z_{trn}}{K}{H} \le \frac{|\mathcal{K}|}{2}\,\info{Z_{trn}}{H} + \sqrt{\frac{|\mathcal{K}|}{2m}}
\end{equation}
This along with the chain rule imply the statement of the theorem. 

\section{Proof of Theorem 4} 
Let $L(\cdot; H)$ be a parametric loss function and write: 
\begin{equation}\label{kappa_1}
\kappa(t) = \PP\Big\{\big|R_{emp}(H; S_m) - R_{true}(H)\big| \ge t\Big\}
\end{equation}
Consider the new pair of hypotheses $(H, K)$, where: 
\begin{equation*}
K = \begin{cases} +1, & \text{ if } R_{emp}(H; S_m) \ge R_{true}(H) + t\\ 
-1, &\text{ if } R_{emp}(H; S_m) \le R_{true}(H) - t \\ 
0, & \text{otherwise}\end{cases} 
\end{equation*} 
Then, by Theorem 3, the uniform generalization risk in expectation for the composition of hypotheses $(H, K)$ is bounded by $4(5/2)\,\info{Z_{trn}}{H} + \sqrt{\frac{\log 3}{2m}}$. This holds uniformly across all parametric loss functions $L'(\cdot; H,K)\to [0,1]$ that satisfy the Markov chain $S_m\to (H, K)\to L'(\cdot; H,K)$. Next, consider the parametric loss: 
\begin{equation*}
L'(Z; H,K) = \begin{cases}L(Z; H) & \text{ if } K=+1\\ 1-L(Z; H) & \text{ if } K=-1\\ 0 &\text{ otherwise} \end{cases} 
\end{equation*}
Note that $L'(\cdot; H,K)$ is parametric with respect to the composition of hypotheses $(H, K)$. Using Eq. \ref{kappa_1}, the generalization risk w.r.t $L'(\cdot; H,K)$ in expectation is, at least, as large as $t\,\kappa(t)$. Therefore, by Theorem 1 and Theorem 3, we have $t\,\kappa(t) \le (5/2)\,\info{Z_{trn}}{H} + \sqrt{\frac{\log 3}{2m}}$. Because $\info{Z_{trn}}{H}\le 1-\mathbb{S}(\LL)$ by definition, the statement of the theorem immediately follows.

\section{Proof of Proposition 3} 
Let $I(X; Y)$ denote the mutual information between $X$ and $Y$ and let $\textbf{H}(X)$ denote the Shannon entropy of the random variable $X$ measured in nats (i.e. using natural logarithms). We write $S_m = (Z_1,\ldots, Z_m)$. We have: 
\begin{align*}
I(S_m;& (H,K)) = \textbf{H}(S_m) - \textbf{H}(S_m\;|\;H,K)\\
&= \sum_{i=1}^m\textbf{H}(Z_i) - \sum_{i=1}^m \textbf{H}(Z_i| H,K, Z_1,\ldots, Z_{i-1})\\ 
&\ge \sum_{i=1}^m \textbf{H}(Z_i) - \textbf{H}(Z_i|H,K)\\ 
&= m I(Z_{trn}; H,K)
\end{align*}
The second line is the chain rule for entropy and the third lines follows from the fact that conditioning reduces entropy. We obtain: 
\begin{equation*}
I(Z_{trn}; H,K) \le \frac{I(S_m; (H,K))}{m}
\end{equation*}
By Pinsker's inequality: 
\begin{equation*}
\info{Z_{trn}}{(H,K)} \le \sqrt{\frac{I(Z_{trn}; (H,K))}{2}}\le  \sqrt{\frac{I(S_m; (H,K))}{2m}}
\end{equation*}
Using the chain rule for mutual information: 
\begin{align*}
\info{Z_{trn}}{(H,K)} &\le  \sqrt{\frac{I(S_m; (H,K))}{2m}} \\ 
&=\sqrt{\frac{I(S_m; H) + I(S_m; K|H)}{2m}} \\ 
&\le \sqrt{\frac{I(S_m; H) + \textbf{H}(K)}{2m}} \\ 
&\le \sqrt{\frac{I(S_m; H) + \log|\mathcal{K}|}{2m}} \\ 
\end{align*}
The desired bound follows by applying the same proof technique of Theorem 4 on the last uniform generalization bound.

\section{Proof of Corollary 1}
First, we note that for any two adjacent samples $S$ and $S'$, we have: 
\begin{equation*}
\PP(H|S) - \PP(H|S') \le (e^\epsilon -1)\;\PP(H|S') + \delta
\end{equation*} 
This follows by definition of differential privacy. Similarly, we have: 
\begin{align*}
\PP(H|S) - \PP(H|S') &\ge (e^{-\epsilon} -1)\;\PP(H|S') -e^{-\epsilon} \delta\\
&= -\Big[(1-e^{-\epsilon})\;\PP(H|S') + e^{-\epsilon}\delta\Big]\\
&\ge -e^\epsilon\;\Big[(1-e^{-\epsilon})\;\PP(H|S') +  e^{-\epsilon}\delta\Big]\\
&= -\Big[(e^\epsilon-1)\PP(H|S') + \delta\Big]
\end{align*} 
Both results imply that: 
\begin{equation}\label{diff_privacy_eq_bound}
\big|\PP(H|S) - \PP(H|S') \big| \le (e^\epsilon-1)\PP(H|S') + \delta
\end{equation}

Writing:
\begin{align*}
\info{Z_{trn}}{H} &= \Dis{\PP(Z_{trn}, H)}{\PP(Z_{trn})\cdot\PP(H)}\\
&=\EE_{Z_{trn}}\;\Dis{\PP(H|Z_{trn})}{\PP(H)}\\
&=\frac{1}{2}\EE_{Z_{trn}}\Big|\Big|\EE_{Z_{trn}'}\big[\PP(H|Z_{trn})-\PP(H|Z_{trn}')\big] \Big|\Big|_1\\
&\le \frac{1}{2}\EE_{Z_{trn},Z_{trn}'}\Big|\Big|\PP(H|Z_{trn})-\PP(H|Z_{trn}') \Big|\Big|_1
\end{align*}
The last inequality follows by convexity. Next, let $S_{m-1}$ be a sample that contains $m-1$ observations drawing i.i.d. from $\PP(z)$. Then: 
\begin{align*}
&\info{Z_{trn}}{H}\le  \frac{1}{2}\EE_{Z_{trn},Z_{trn}'}\Big|\Big|\PP(H|Z_{trn})-\PP(H|Z_{trn}') \Big|\Big|_1\\
&\;= \frac{1}{2}\EE_{Z_{trn},Z_{trn}'}\Big|\Big|\EE_{S_{m-1}}\big[\PP(H|Z_{trn}, S_{m-1})-\PP(H|Z_{trn}', S_{m-1})\big] \Big|\Big|_1\\
&\le\;\frac{1}{2} \EE_{S,S'} \Big|\Big|\PP(H|S)-\PP(H|S') \Big|\Big|_1,
\end{align*}
where $S,S'$ are two adjacent samples.

Next, we expand the $\ell_1$ distance and use Eq \ref{diff_privacy_eq_bound}: 
\begin{align*}
\info{Z_{trn}}{H} &\le \frac{1}{2} \EE_{S,S'} \Big|\Big|\PP(H|S)-\PP(H|S') \Big|\Big|_1 \\ 
&= \frac{1}{2} \EE_{S,S'} \sum_{H\in\mathcal{H}}\big|\PP(H|S)-\PP(H|S') \big|\\
&\le \frac{1}{2} \EE_{S,S'} \sum_{H\in\mathcal{H}}\big[(e^\epsilon-1) \PP(H|S') + \delta \big]\\
&= \frac{e^\epsilon-1+\delta}{2}
\end{align*}
Finally, the desired bound follows by combining the last inequality with Theorem 4.

\section{Proof of Corollary 2}
It has been shown in \cite{alabdulmohsin_nips_2015} that the supremum generalization risk is achieved (arbitrarily well) using the following binary-valued loss: 
\begin{equation}\label{max_gen_equaion}
L^\star(z; H) = \mathbb{I}\big\{\PP(Z_{trn}=z|H)\ge \PP(Z_{trn}=z) \big\} 
\end{equation}
Therefore, if an algorithm is $(\epsilon,\delta)$ robustly generalizing, let the adversary $\mathcal{A}$ (or equivalently the parametric loss $L(\cdot; H)$) be fixed to the one given by Eq. \ref{max_gen_equaion}. Hence, we have by definition of robust generalization: 
\begin{equation}
\PP\Big\{\big|\EE_{Z\sim \PP(z)} L^\star(Z; H) - \frac{1}{m} \sum_{Z_i\in S} L^\star(Z_i; H)\big| \le \epsilon\Big\}\ge 1-\delta,
\end{equation}
Therefore: 
\begin{equation*}
\Big|\mathbb{E}_{S, H}\big[ \EE_{Z\sim \PP(z)} L^\star(Z; H) - \frac{1}{m} \sum_{Z_i\in S} L^\star(Z_i; H)\big]\Big|\le \epsilon+\delta
\end{equation*}
Because $L^\star(\cdot; H)$ achieves the maximum possible generalization risk in expectation \citep{alabdulmohsin_nips_2015}, we have the uniform generalization bound $\info{Z_{trn}}{H} \le \epsilon+\delta$. Hence, ($\epsilon,\delta)$ robust generalization implies a uniform generalization at the rate $\epsilon+\delta$. 

The proof of the converse follows from our concentration bound in Theorem 4, which shows that uniform generalization in expectation implies a generalization in probability. In particular, any algorithm that generalizes uniformly with rate $\tau$ is ($\epsilon, \gamma$) robustly generalizing, with $\gamma = (5/2)(\tau + \sqrt{\log 9/(25m)})/\epsilon$.

\section{Proof of Theorem 5}
Before we prove the statement of the theorem, we begin with the following lemma:
\begin{lemma}\label{tight_lemma}
Let the observation space $\ZZ$ be the interval $[0,1]$, where $\PP(z)$ is continuous in $[0,1]$. Let $H\subseteq S_m: |H|=k$ be a set of $k$ examples picked at random without replacement from the sample $S_m$. Then $\info{Z_{trn}}{H}=\frac{k}{m}$. 
\end{lemma} 
\begin{proof}
First, we note that $\PP(Z_{trn}|H)$ is a \emph{mixture} of two distributions: one that is uniform in $H$ with probability $k/m$, and the original distribution $\PP(z)$ with probability $1-k/m$. By Jensen's inequality, we have $\info{Z_{trn}}{H}\le k/m$. Second, let the parametric loss be $L(\cdot; H)=\mathbb{I}\{Z\in H\}$. Then, $|R_{gen}(\LL)|=\frac{k}{m}$. By Theorem 1, we have $\info{Z_{trn}}{H}\ge |R_{gen}(\LL)|=k/m$. Both bounds imply the statement of the lemma. 
\end{proof}

Now, we prove Theorem 5. Consider the example where $\mathcal{Z}=[0,\,1]$ and suppose that the observations $Z\in\mathcal{Z}$ have a continuous marginal distribution. Because $t$ is a rational number, let the sample size $m$ be chosen such that $k=t\,m$ is an integer. 

Let $\{Z_1,\ldots, Z_m\}$ be the training set, and let the hypothesis $H$ be given by $H=\{Z_1,\ldots, Z_k\}$ with some probability $\delta>0$ and $H=\{\}$ otherwise. Here, the $k$ instances $Z_i\in H$ are picked uniformly at random without replacement from the sample $S_m$. To determine the variational information between $Z_{trn}$ and $H$, we consider the two cases:
\begin{enumerate} 
\item If $H\neq \{\}$, then $\Dis{\PP(Z_{trn})}{\PP(Z_{trn}|H)}=t$ as proved in Lemma 1. This happens with probability $\delta$ by construction. 
\item If $H=\{\}$ then $\PP(Z_{trn}|H)=\PP(Z_{trn})$. Hence, we have $\Dis{\PP(Z_{trn})}{\PP(Z_{trn}|H=\{\})}=0$. This happens with probability $1-\delta$. 
\end{enumerate} 

So, by combining the two cases above, we deduce that: 
\begin{equation*}
\info{Z_{trn}}{H} = \EE_H\; \Dis{\PP(Z_{trn})}{\PP(Z_{trn}\;|\;H)} = t\; \delta.
\end{equation*}
Therefore, $\LL$ generalizes uniformly with the rate $t\delta$. Next, let the loss $L(\cdot; H)$ be given by $L(Z\;;\;H) = \mathbb{I}\big\{Z\in H\big\}$. With this loss: 
\begin{align*}
\PP\Big\{\big|R_{emp}(H; S_m) - R_{true}(H)\big| = t \Big\} &= \delta\\
&=\frac{\info{Z_{trn}}{H}}{t},
\end{align*}
which is the statement of the theorem. 

\section{Proof of Proposition 4}
This proposition is proved in Corollary 1.

\bibliographystyle{apalike}
\bibliography{chain_learn}

\begin{thebibliography}{}

\bibitem[Alabdulmohsin, 2015]{alabdulmohsin_nips_2015}
Alabdulmohsin, I. (2015).
\newblock Algorithmic stability and uniform generalization.
\newblock In {\em {NIPS}}, pages 19--27.

\bibitem[Audibert and Bousquet, 2007]{audibert2007combining}
Audibert, J.-Y. and Bousquet, O. (2007).
\newblock Combining {PAC-B}ayesian and generic chaining bounds.
\newblock {\em {JMLR}}, 8:863--889.

\bibitem[Bartlett and Mendelson, 2002]{bartlett2002rademacher}
Bartlett, P.~L. and Mendelson, S. (2002).
\newblock Rademacher and gaussian complexities: Risk bounds and structural
  results.
\newblock {\em {JMLR}}, 3:463--482.

\bibitem[Bassily and Freund, 2016]{typicality_based_stability}
Bassily, R. and Freund, Y. (2016).
\newblock Typicality based stability and privacy.
\newblock arXiv:1604.03336 [cs.LG].

\bibitem[Bassily et~al., 2016]{bassily2015adaptive}
Bassily, R., Nissim, K., Smith, A., Steinke, T., Stemmer, U., and Ullman, J.
  (2016).
\newblock Algorithmic stability for adaptive data analysis.
\newblock In {\em Proceedings of the Forty-Seventh Annual ACM on Symposium on
  Theory of Computing {(STOC)}}, pages 1046--1059.

\bibitem[Blumer et~al., 1989]{blumer1989learnability}
Blumer, A., Ehrenfeucht, A., Haussler, D., and Warmuth, M.~K. (1989).
\newblock Learnability and the {V}apnik-{C}hervonenkis dimension.
\newblock {\em Journal of the {ACM} (JACM)}, 36(4):929--965.

\bibitem[Boucheron et~al., 2004]{cocentrationboucheron}
Boucheron, S., Lugosi, G., and Bousquet, O. (2004).
\newblock Concentration inequalities.
\newblock In {\em Advanced Lectures on Machine Learning}, pages 208--240.
  Springer.

\bibitem[Bousquet et~al., 2004]{stat_learn_theory_2004}
Bousquet, O., Boucheron, S., and Lugosi, G. (2004).
\newblock Introduction to statistical learning theory.
\newblock In Bousquet, O., von Luxburg, U., and R\"atsch, G., editors, {\em
  Advanced Lectures on Machine Learning}, volume 3176, pages 169--207.

\bibitem[Bousquet and Elisseeff, 2002]{bousquet2002stability}
Bousquet, O. and Elisseeff, A. (2002).
\newblock Stability and generalization.
\newblock {\em {JMLR}}, 2:499--526.

\bibitem[Cover and Thomas, 1991]{cover2012elements}
Cover, T.~M. and Thomas, J.~A. (1991).
\newblock {\em Elements of information theory}.
\newblock Wiley \& Sons.

\bibitem[Csisz\'ar, 1972]{f_divergence_csiszar}
Csisz\'ar, I. (1972).
\newblock A class of measures of informativity of observation channels.
\newblock {\em Periodica Mathematica Hungarica}, 2:191--213.

\bibitem[Csisz\'ar, 2008]{f_divergence_csiszar_2008}
Csisz\'ar, I. (2008).
\newblock Axiomatic characterizations of information measures.
\newblock {\em Entropy}, 10:261--273.

\bibitem[Cummings et~al., 2016]{robust_gen_2016}
Cummings, R., Ligett, K., Nissim, K., Roth, A., and Wu, Z.~S. (2016).
\newblock Adaptive learning with robust generalization guarantees.
\newblock In {\em Proceedings of the 29th Annual Conference on Learning Theory
  {(COLT)}}.

\bibitem[David et~al., 2016]{supervised_compression}
David, O., Moran, S., and Yehudayoff, A. (2016).
\newblock Supervised learning through the lens of compression.
\newblock In {\em {NIPS}}.

\bibitem[Dwork et~al., 2015]{adaptive_learning_2015}
Dwork, C., Feldman, V., Hardt, M., Pitassi, T., Reingold, O., and Roth, A.
  (2015).
\newblock Preserving statistical validity in adaptive data analysis.
\newblock In {\em Proceedings of the Forty-Seventh Annual ACM on Symposium on
  Theory of Computing {(STOC)}}, pages 117--126.

\bibitem[Dwork and Roth, 2013]{dwork2013algorithmic}
Dwork, C. and Roth, A. (2013).
\newblock The algorithmic foundations of differential privacy.
\newblock {\em Theoretical Computer Science}, 9(3-4):211--407.

\bibitem[Hardt et~al., 2016]{sgd_train_fast_2015}
Hardt, M., Recht, B., and Singer, Y. (2016).
\newblock Train faster, generalize better: Stability of stochastic gradient
  descent.
\newblock In {\em Proceedings of the 33rd international conference on Machine
  learning {(ICML)}}.

\bibitem[Luntz and Brailovsky, 1969]{luntz1969estimation}
Luntz, A. and Brailovsky, V. (1969).
\newblock On estimation of characters obtained in statistical procedure of
  recognition.
\newblock {\em Technicheskaya Kibernetica}, 3(6).

\bibitem[McAllester, 2003]{mcallester2003pac}
McAllester, D. (2003).
\newblock {PAC-B}ayesian stochastic model selection.
\newblock {\em Machine Learning}, 51:5--21.

\bibitem[Reid and Williamson, 2009]{pinsker_colt_2009}
Reid, M.~D. and Williamson, R.~C. (2009).
\newblock Generalised {P}insker inequalities.
\newblock In {\em {COLT}}.

\bibitem[Russo and Zou, 2016]{control_bias_it_2016}
Russo, D. and Zou, J. (2016).
\newblock Controlling bias in adaptive data analysis using information theory.
\newblock In {\em Proceedings of the 19th International Conference on
  Artificial Intelligence and Statistics {(AISTATS)}}.

\bibitem[Shalev-Shwartz and Ben-David, 2014]{SSS2014understandML}
Shalev-Shwartz, S. and Ben-David, S. (2014).
\newblock {\em Understanding Machine Learning: From Theory to Algorithms}.
\newblock Cambridge University Press.

\bibitem[Shalev-Shwartz et~al., 2010]{shalev2010learnability}
Shalev-Shwartz, S., Shamir, O., Srebro, N., and Sridharan, K. (2010).
\newblock Learnability, stability and uniform convergence.
\newblock {\em {JMLR}}, 11:2635--2670.

\bibitem[Stigler, 1986]{stigler1986history}
Stigler, S.~M. (1986).
\newblock {\em The history of statistics: The measurement of uncertainty before
  1900}.
\newblock Harvard University Press.

\bibitem[Vapnik and Chapelle, 2000]{vapnik2000bounds}
Vapnik, V. and Chapelle, O. (2000).
\newblock Bounds on error expectation for support vector machines.
\newblock {\em Neural Computation}, 12(9):2013--2036.

\bibitem[Vapnik, 1999]{vapnik1999overview}
Vapnik, V.~N. (1999).
\newblock An overview of statistical learning theory.
\newblock {\em Neural Networks, {IEEE} Transactions on}, 10(5):988--999.

\end{thebibliography}

\end{document}